\documentclass[runningheads]{llncs}
\usepackage[T1]{fontenc}
\usepackage{amsmath}
\usepackage{amssymb}
\usepackage{amsfonts}
\usepackage{algorithm}
\usepackage{algpseudocode}

\usepackage{graphicx}
\begin{document}
\title{Coverage Path Planning For Minimizing Expected Time to Search For an Object With Continuous Sensing}
%
%
\author{Linh Nguyen\orcidID{0009-0009-3518-929X}}
\authorrunning{Linh Nguyen}
%
\institute{Stony Brook University, Stony Brook NY 11794, USA\\ \email{linh.nguyen.1@stonybrook.edu}}
\maketitle              
\begin{abstract}
We present several results of both theoretical as well as practical interests in optimal search theory. First, we propose the quota lawn mowing problem, an extension of the classic lawn mowing problem in computational geometry, as follows: given a quota of coverage, compute the shortest lawn mowing route to achieve said quota. We give constant-factor approximations for the quota lawn mowing problem.

Second, we investigate the expected detection time minimization problem in geometric coverage path planning with local, continuous sensory information. We provide the first approximation algorithm with provable error bounds that runs in pseudopolynomial time. Our ideas also extend to another search mechanism, namely visibility-based search, which is related to the watchman route problem. We complement our theoretical analysis with some simple but effective heuristics for finding an object in minimum expected time, on which we provide simulation results. 

\end{abstract}
\section{Introduction and Related Work}
In this work, we study search and path planning problems in \textit{geometric domains}. Consider, for example, a remote-controlled camera inspecting a manufactured component for porosities, or a submarine using radar/sonar to search for a missing object, etc. We model this search, where the agent needs to be in proximity to the target for the sensor to detect it, by the \textit{lawn mowing problem}, a classic computational geometry problem posed by Arkin, Fekete, and Mitchell in~\cite{arkin2000approximation}. Given a planar region, compute a minimum-length route along which a circle or square (``mower'') can be slid so that every point in the region is ``mowed''. We investigate a natural generalization of the lawn mowing problem, in which we are given a \textit{quota} of coverage, and we are to compute the shortest route to achieve said quota. This is relevant in time-sensitive or resource-constrained scenarios, where complete coverage may not be feasible or necessary. Another related problem, which shares a similar detection mechanism with the lawn mowing problem is the honey-pot search problem \cite{dasgupta2006honey}, in which a searcher must maximize the probability of finding a target within the given time budget.

The equivalence of the lawn mowing problem for visibility-based search is the watchman route problem, introduced by Chin and Ntafos~\cite{chin1986optimum}, where the agent needs to see the target:``seeing'' means establishing a line of vision that lies entirely within the domain. Given a polygon, the watchman route problem asks for a minimum-length route that sees the whole polygon. 

Both the lawn mowing problem and the watchman route problem, when stated as search problems, are concerned with the \textit{worst-case} guarantee, as are many previous studies in coverage path planning, see the surveys \cite{cabreira2019survey,tan2021comprehensive}. We study the \textit{average-case}, i.e., we seek to optimize \textit{expected} duration of the search, which is more beneficial in the long run if the search is to be carried out on a regular basis. Several studies have looked at minimizing expected detection time in a graph. In \cite{berman2011optimal}, Bernman, Ianovsky and Krass proposed an exponential-time dynamic programming algorithm for finding a uniformly distributed target, which Teller, Zofi and Kaspi later extended to non-uniform distributions while also presenting a branch-and-bound algorithm along with several heuristics~\cite{teller2019minimizing}.  Sarmiento, Murrieta and Hutchinson were the first to investigate the expected detection time in visibility-based search~\cite{sarmiento2003efficient,sarmiento2004planning}. They proposed a heuristic that selects a set of points that collectively see the whole polygon (known as ``guards'') and visits those points in an ordering resulting from minimizing a utility function.

The problem of minimizing expected detection time is also related to the minimum latency problem (also referred to in the literature as the traveling repairman problem). Given some nodes in a graph or some points in the plane, compute a route that minimizes the average delay of all the nodes. Several approximation algorithms are known~\cite{archer2008faster,arora2003approximation,goemans1998improved}, including a polynomial-time approximation scheme~\cite{sitters2014polynomial}. Our problem is a continuous minimum-latency type problem. 

Most search problems involving path planning in a geometric domain, or even a discrete graph theoretic domain are known to be NP-hard~\cite{huynh2024optimizing,trummel1986complexity}. Optimal search is an overarching field of research that arises in numerous practical applications, including but not limited to: surveillance, robotics, search-and-rescue, etc. The theoretical basis for search theory was introduced in the seminal work of Koopman~\cite{koopman1956theory1,koopman1956theory2,koopman1957theory}. Depending on the task, perception, and target modeling, along with the objective function to be optimized, many different variations have been proposed and studied. We refer the reader to the surveys~\cite{chung2011search,hohzaki2016search,raap2019moving,sato2008path,yu2019survey} for a comprehensive review.

We present several novel results on optimal 
 search in a geometric setting:
\begin{enumerate}
    \item For the quota lawn mowing problem, we provide theoretical approximation algorithms, the running time of which depends on the approximation algorithm for quota-TSP/MST involved.
    \item We prove that minimizing the expected detection time is NP-hard in a simple polygon for both search mechanisms, lawn mowing and visibility-based search. Previously, only hardness in polygons with holes was known~\cite{sarmiento2003efficient}.
    \item We provide the first pseudopolynomial-time approximation algorithm with provable error bounds for minimizing the expected detection time for both search mechanisms.
    
\end{enumerate}

\section{Preliminaries}
We model the area to be searched by a closed planar region $R$, a subset of $\mathbb{R}^2$. If $R$ is connected and the boundary of $R$, $\partial R$, consists of a finite number of non-crossing line segments, then $R$ is a \textit{polygonal region}, or synonymously, a \textit{polygon with holes}. If $R$ has no holes, i.e., $R$ is simply connected, then we call $R$ a \textit{simple polygon}. We assume $R$ is a polygon with holes with $n$ vertices, though our results can be generalized to more general regions (with regular curved arcs on the boundary or multiple connected components).

We are given a pointwise \textit{robot/cutter}. Throughout this paper, we shall use the term ``cutter'' when discussing the lawn mowing problem and the term ``robot'' when discussing searching. Assume that the operational area of the robot/cutter, $\chi$, is either an axis-aligned square or a circle centered on it. Without loss of generality, we can scale $R$ so that $\chi$ is a unit square (side length 1) or a unit circle (radius 1). If the robot/cutter traverses a trajectory $\gamma$, the \textit{area of coverage}, denoted $C(\gamma)$, is the Minkowski sum of $\gamma$ with the operational area: $C(\gamma) = (\gamma + \chi) \cap R$. Note that, in the lawn mowing problem, $\gamma$ need not stay within $R$. See Figure~\ref{fig:area_of_coverage} for an illustration.

\begin{figure}[h]
        \centering
        \includegraphics[width=0.75\textwidth]{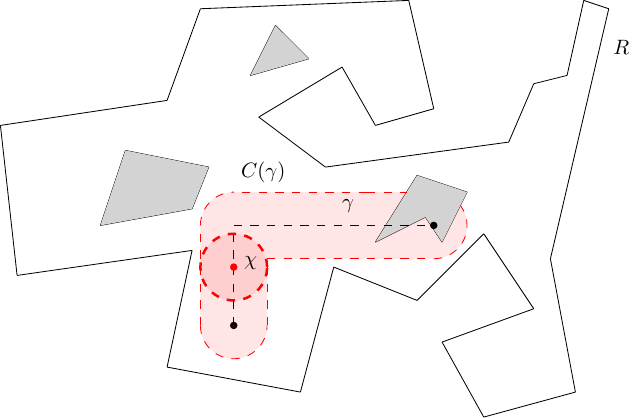}
        \caption{The area of coverage $C(\gamma)$ (red) by a robot with circular operational area $\chi$ traveling along $\gamma$.}
        \label{fig:area_of_coverage}
\end{figure}

We use $|\cdot|$ to denote Euclidean measure of geometric objects (e.g., length or area), as well as cardinality of finite sets, the meaning should be clear from the context. First, we consider the following problem:
\paragraph{\textbf{The quota lawn mowing problem}}Given an area quota $0 \le A \le |R|$, compute the shortest tour (cycle) $\gamma$ such that $|C(\gamma)| \ge A$.

This can be seen as a generalization of the lawn mowing problem, which is NP-hard even if the input region is a simple polygon \cite{arkin2000approximation}. Moreover, the lawn mowing problem is algebraically hard~\cite{feketecircling}. The same hardness for the quota lawn mowing problem follows naturally.

We consider how the lawn mower analogy can model some geometric optimal search problems. Suppose there is a \textit{target} located within $R$. We assume that the location of the target is uniformly distributed, though our results can be easily adapted to non-uniform distributions, that is, when we have prior knowledge of some regions more likely to contain the target, by simply replacing area with probability measure. A robot, given a starting point $s$ and equipped with a sensor, looks for the target. Think of a camera scanning a manufactured part for a small dent, or a submarine using radar/sonar to look for a missing ship/airplane on the ocean floor.

If the sensor can detect the target within a circle or square centered around the robot, we call it the \textit{lawn mowing search mechanism}. If the sensor is visibility-based with infinite range, we have the classic watchman route problem (in which the geometry of the region $R$ not only limits visibility but also movements), which we refer to as \textit{visibility-based search mechanism}. 

The lawn mowing problem and the watchman route problem are concerned with the worst-case scenario of the search process, where the objective is to minimize the distance traveled to guarantee detection of the target, i.e. every point in the domain must be covered. The quota lawn mowing problem introduced earlier, can be interpreted as that of computing a minimum route to guarantee a certain probability of detection. In many applications, we are also interested in an average-case analysis, thus we study the \textit{expected detection time} as the objective function to optimize. Suppose the robot follows a trajectory $\gamma$ such $C(\gamma) = R$. We define the random variable $T$ to be the time at which the target is detected by the robot. We assume that the robot travels at a constant speed of 1, without loss of generality, so that distance and time are equivalent. Denote by $C(\gamma, t)$ the partial area of coverage accumulated by the robot by time $t$ as it travels along $\gamma$ (note that $C(\gamma, |\gamma|) = C(\gamma) = R$). Then the probability of detection at time $t$ is equal to the fraction of area covered: $\displaystyle\text{Pr}_\gamma(T \le t) = \frac{|C(\gamma,t)|}{|R|}.$
Thus, since $T > 0$, the expected value of $T$ can be written as
\begin{align*}
    E[T\mid\gamma] &= \int_0^\infty\text{Pr}_\gamma(T > t)dt = \int_0^\infty\left(1 - \frac{|C(\gamma,t)|}{|R|}\right)dt.
\end{align*}

\paragraph{\textbf{The expected detection time minimization problem}}Compute a route $\gamma$ that minimizes $E[T\mid\gamma]$.

    \label{fig:diff_min_length_min_expectation}

\section{Lawn Mowing Problem with a Quota}
We begin by assuming the operational area of the cutter, $\chi$, is a closed unit square. As the cutter moves, $\chi$ translates but does not rotate.

To achieve a quota of coverage, the cutter must sweep over every point in a subset of the domain $R$. The uncountably infinite number of points that the cutter must cover makes for a computational challenge. We overcome this challenge (at the cost of an approximation factor) by an appropriate discretization of $R$. We denote by $\mathcal{P}$ the set of pixels whose vertices are integer grid points that have a non-empty intersection with $R$. If the cutter is given a starting point $s$, the grid inducing $\mathcal{P}$ is shifted if necessary so that $s$ is at the center of a pixel. Denote by $\mathcal{G}$ the \textit{dual grid graph} of $\mathcal{P}$. Each node in $\mathcal{G}$ corresponds to the center point of a pixel in $\mathcal{P}$, two nodes are adjacent if their pixels share a side edge (Figure~\ref{fig:dual_grid}). Let $N = |\mathcal{P}|$, then we can identify $\mathcal{G}$ and $\mathcal{P}$ in $O(N + n\log n)$ from a description of $R$ as an $n$-vertex polygonal domain.

We consider the input complexity to depend on the \textit{magnitude} (the bit complexity) of the coordinates of the vertices of $R$. This allows for algorithms whose running time is \textit{pseudopolynomial} (depending on the input size as well as magnitude of the input). This is because even for a relatively combinatorially simple region, a quota lawn mowing tour can have a very large number of turns. This was first observed in \cite{arkin2000approximation}.

\begin{figure}[h]
        \centering
        \includegraphics[width=0.6\textwidth]{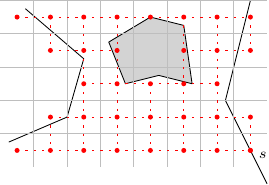}
        \caption{The pixels $\mathcal{P}$ (grey) and the dual grid graph $\mathcal{G}$ (red).}
        \label{fig:dual_grid}
\end{figure}

Let $\gamma^*$ denote an optimal lawn mowing tour for a given quota of coverage $A$. To approximate $\gamma^*$, we require the following lemma.

\begin{lemma}
\label{lem:discretize_to_grid}
    For a unit square cutter, there exists a tour $\gamma_{\mathcal{G}}$ whose vertices are a subset of the vertices of $\mathcal{G}$ such that $|\gamma_{\mathcal{G}}| \le O(1)|\gamma^*|$ and $C(\gamma^*) \subseteq C(\gamma_{\mathcal{G}})$.
\end{lemma}
\begin{proof}
    Notice that $\gamma^*$ is necessarily the shortest (fully covering) lawn mowing tour of $C(\gamma^*)$, a subregion of $R$. According to \cite[Theorem 3]{arkin2000approximation}
    \begin{enumerate}
   
        \item[(i)] If the cutter is limited to rectilinear motion ($\gamma^*$ consists of only axis-aligned segments), there exists $\gamma_{\mathcal{G}}$ within $\mathcal{G}$ of length no greater than $3|\gamma^*|$ covering $C(\gamma^*)$.
        \item[(ii)] If the cutter can translate arbitrarily, we augment $\mathcal{G}$ with diagonals: join two center points with a diagonal (45 degrees slope) if they are separated by distance $\sqrt{2}$. There exists $\gamma_{\mathcal{G}}$ within the diagonal-augmented grid of length no greater than $\frac{6}{\sqrt{2 + \sqrt{2}}}|\gamma^*|$ covering $C(\gamma^*)$.
    \end{enumerate}
    When $|\gamma^*|$ is small, we can employ the strategy for short covering tours in \cite{arkin2000approximation}, the details of which we omit in this version.
    \hfill$\blacksquare$
\end{proof}
An algorithm that computes an optimal (quota) lawn mowing tour will generally need to select turning points from a continuum rather than a discrete, finite set. This presents significant computational challenges (recall the algebraic hardness of the problem). Lemma \ref{lem:discretize_to_grid} establishes that, for purposes of constant-factor approximations, it suffices to discretize the problem into a grid graph.

 For each node $p$ of $\mathcal{G}$, the center point of the pixel dual to $p$, let $r(p)$ be the area of the pixel that lies inside $R$, i.e., $r(p) = |p\cap R|$. This allows us to treat the lawn mowing problem with a quota as an instance of the quota TSP (a special case of the Prize Collecting TSP~\cite{ausiello2018prize}):

 \paragraph{\textbf{The quota TSP}} Find the shortest tour through a subset of pixel center points such that the total reward of the tour is no smaller than 
 $A$. 
 
 By Lemma \ref{lem:discretize_to_grid}, a $c$-approximation to the quota TSP yields a $3c$-approximation in the case of rectilinear motion and a $\frac{6c}{\sqrt{2 + \sqrt{2}}}$-approximation in the case of arbitrary motion, to the quota lawn mowing problem.

 We can approximate the quota TSP problem with a factor of 5 using the methods in \cite{ausiello2018prize} with polynomial (in $N$) running time. Further, we can achieve a factor of $(1 + \varepsilon)$ for any $\varepsilon > 0$ by adapting the methods for geometric $k$-TSP (Arora \cite{arora1996polynomial} or Mitchell \cite{mitchell1999guillotine}), a special case of quota TSP where the reward of every point is unary and the quota is $k$ as follows:
 \begin{itemize}
     \item[(1)] Scale the rewards and quota to be integers. Specifically, let $M = 10^D$, where $D$ is the maximum number of decimal digits used to represent the fraction part of any $r(p)$ or $A$. Then $M = D\log_2(10) = O(B)$, where $B$ is the bit complexity of the problem instance.
     
     We set $\overline{A}:=MA$ and $\overline{r}(p):=Mr(p)$ for every $p\in\mathcal{P}$.
     \item[(2)] For each $p$, create $\overline{r}(p)$-many duplicates of $p$ at the same location as $p$ (distance 0 away from $p$). This create an $O(NM) = O(NB)$ instance of $k$-TSP.
     \item[(3)] Approximate the optimal $k$-TSP tour for the instance created with $k=\overline{A}$.
 \end{itemize}
 The running time of this approach is $(NB)^{O(1/\varepsilon)}$.
 
 Similarly, we can use the same transformation to the $k$-MST problem: Find the minimum cost tree spanning at least $k$ vertices. Doubling an $\alpha$-approximate tree (traversing each edge twice in opposite directions) gives a $2\alpha$-approximate tour. Several approximations with small factors are known for the $k$-MST problem, the current best factor is $2$ for general graphs \cite{garg2005saving} and $(1 + \varepsilon)$ for points in the plane \cite{arora1996polynomial,mitchell1999guillotine}.
 \paragraph{Circular cutter} If the cutter is a circle, rather than square pixels, we consider hexagonal pixels, each of diameter 2. We join two center points by an edge if their respective hexagons share part of their boundary, creating a triangular grid graph (Figure~\ref{fig:hexagonal_grid}).

\begin{figure}[h]
        \centering
        \includegraphics[width=0.6\textwidth]{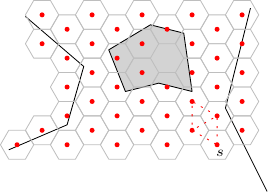}
        \caption{The hexagonal pixels for the case of a circular cutter.}
        \label{fig:hexagonal_grid}
\end{figure}
A constant-factor approximation to an optimal quota lawn mowing tour $\gamma^*$, can once again be found within the triangular grid graph. Specifically, one no longer than $2\sqrt{3}|\gamma^*|$ (see \cite[Theorem 3]{arkin2000approximation}). The remainder of the approximation method follows through like the case with square cutter.

\section{Minimizing Expected Time to Search For an Object}
We turn our attention to the problem of minimizing expected search time, which has numerous practical implications. The findings on the quota lawn mowing problem while are of great independent interests, also allow us to find approximation algorithms with provable guarantees for the expected search time minimization problem.
\subsection{Lawn mowing search mechanism}
First, we rewrite the formula for the expected value of $T$, the amount of time (or distance, interchangably, because we assume the robot travels at a fixed speed of 1) before the target is detected by the robot as follows:
\begin{align*}
    E[T\mid\gamma]& = \int_0^\infty\left(1 - \frac{|C(\gamma,t)|}{|R|}\right)dt= \frac{1}{|R|}\int_0^\infty\left(\int\limits_{x\notin C(\gamma, t)}dx\right)dt.
\end{align*}
For the purpose of optimizing expected detection time in a given region $R$, we can drop the denominator $|R|$. 
Notice that for each infinitesimal area element $dx \in R$, the integral above is taken over the region $dx \in R \setminus C(\gamma,t)$. Hence, we can reverse the order of integration to be
\begin{align*}
    \int\limits_{x\in R}\left(\int_0^{t:x\notin C(\gamma,t)}dt\right)dx=\int\limits_{x\in R}t_\gamma(x)dx.
\end{align*}
where $t_\gamma(x)$ is the time when the accumulated area of coverage of the robot first contains $x$ when traveling along $\gamma$. It may be more intuitive to consider this formula in the discrete case, where the location of the target is given as a uniform distribution on $m$ discrete points $x_1, \ldots, x_m$. Suppose $\gamma$ visits those points in that order, then
\begin{align*}
    E[T\mid \gamma] &= \sum_{i=1}^mt_\gamma(x_i)\text{Pr}(\text{target is found when $\gamma$ reaches }x_i)= \frac{1}{m}\sum_{i=1}^mt_\gamma(x_i).
\end{align*}
In fact, $t_\gamma(x_i)$ is known as the \textit{latency} of $x_i$.

\begin{theorem}
    Minimizing $E[T\mid\gamma]$ is NP-hard, even in a simple polygon.
    \label{thm:min_E_NP_hard}
\end{theorem}

\begin{proof}
    Our reduction is from the Minimum Latency problem: Given $n$ points $p_1, \ldots, p_n$ on the Euclidean plane, compute a tour $\gamma$ that minimizes $\sum\limits_{i=1}^{n}t_\gamma(p_i)$, where $t_\gamma(p_i)$ is the distance traveled before first visiting $p_i$. The Minimum Latency problem is known to be as hard as the Euclidean TSP~\cite{blum1994minimum}. Our reduction works for both rectilinear and unrestricted movement, square and circular searcher.
    Given a Minimum Latency problem instance, scale the instance so that the side length (or radius) of the area of coverage of the robot is some $\varepsilon$ infinitesimally small. We construct a simple polygon $P$ as shown in Figure~\ref{fig:reduction_from_MLP}. First, interconnect the points by a tree (e.g., the minimum spanning tree). Surround each point with a small square room of side-length $\varepsilon$. Each edge of the tree is a narrow hallway of width $\varepsilon^2$.
    \begin{figure}[h]
        \centering
        \includegraphics[width=0.75\textwidth]{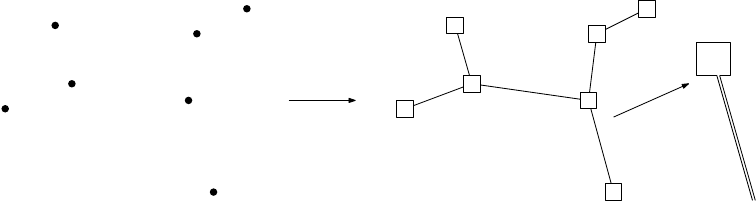}
        \caption{The construction used in showing NP-hardness of minimizing $E[T\mid\gamma]$.}
        \label{fig:reduction_from_MLP}
    \end{figure}
    
    Recall that in the lawn mowing search mechanism, the movement of the robot is not constrained to stay within $P$. Consequently, an expected detection time minimizing tour prioritizes the rooms before attempting to cover all the area of the hallways. Such a tour must visit the rooms in the order that a latency-minimizing tour visits the corresponding points. \hfill$\blacksquare$
\end{proof}
\subsubsection{Approximation algorithm} We give an approximation algorithm for minimizing $E[T\mid\gamma]$, utilizing the solution to the quota lawn mowing problem as a subroutine. In the reduction from the quota lawn mowing problem to the $k$-TSP, we multiplied the area of each pixel with $M = 10^D$, where $D$ is the decimal precision, thus effectively multiplied the area of $R$ by $M$. Recall that $M = O(B)$, where $B$ is the bit complexity.

For any $A\in\{1, 2, \ldots, M|R|\}$, let $L(A)$ be length of the shortest tour that covers an area of $A$, which means a $c$-approximation algorithm would give a tour no longer than $cL(A)$ for area quota $A$. For $j= 1, \ldots, O(\log N)$ (suppose $N > 1$, else the problem is trivial), denote by $A_j$ the maximum value such that $L(A_j) \le 2^j$. For $j= 1, \ldots, O(\log N)$, we can find $\overline{A_j}$, the maximum area quota for which the $c$-approximation returns a tour no longer than $c2^j$, by a binary search on $\{1, 2, \ldots, M|R|\}$ as the input quota to the $c$-approximation. Clearly, $A_j \le \overline{A_j}$.

We state our approximation algorithm for finding a route that minimizes expected detection time as follows: For $j = 1, \ldots, O(\log N)$, follow the tour (returned by the $c$-approximation) associated with area quota $\overline{A_j}$, return to the starting point before each increment of $j$. We denote this tour by $\overline{\gamma}$.

\subsubsection{Running time} The running time of the above approximation algorithm is the running time of the $c$-approximation to the quota lawn mowing problem (which depends on the quota TSP/MST or $k$-TSP/MST heuristic employed) multiplied by $O(\log N)O(\log M|R|) = O(\log N(B + \log N))$.

\subsubsection{Analyzing the approximation factor}
Let $\gamma^*$ be an expected detection time optimal tour, and let $\text{Pr}_{\gamma^*}(T > t) = f_{\gamma^*}(t)$. Specifically, $f_{\gamma^*}(t)$ is the uncovered fraction of area after the robot follows along $\gamma^*$ for a time/distance of $t$. Note that $f_{\gamma^*}(t)$ is continuous and monotone (decreasing).

We show another way to compute $E[T\mid \gamma^*]$, which will be helpful in our analysis of the approximation factor. Draw the graph of $f_{\gamma^*}(t)$ on the Cartesian plane, we wish to compute the area bounded by its curve and the 2 axes. If we switch the 2 axes and consider $t$ a function of $p$, the area under the curve stays the same. For any $0 \le 1 \le p$, define the function $f_{\gamma^*}'(p) = t$, where $t$ is the minimum latency along $\gamma^*$ so that the uncovered fraction of area is no smaller than $p$. Then
\[\int_0^{|\gamma^*|}f_{\gamma^*}(t)dt = \int_0^1f_{\gamma^*}'(p)dp.\]
See Figure~\ref{fig:integral} for an illustration.

\begin{figure}[h]
        \centering
        \includegraphics[width=\textwidth]{
        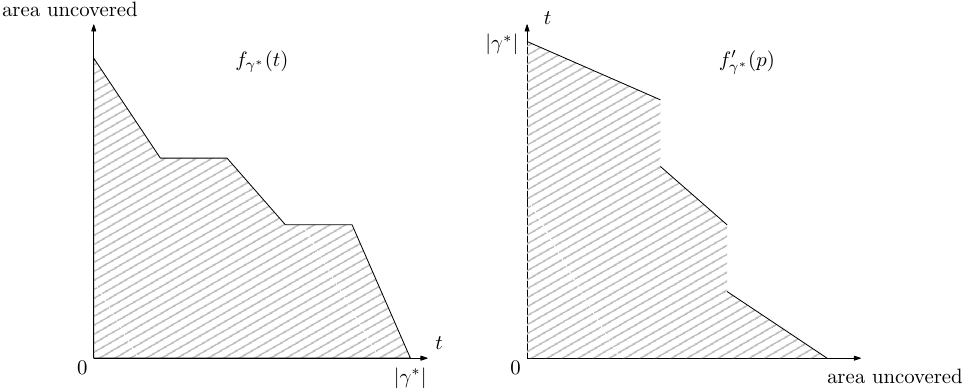}
        \caption{The graph of $f(t)$ (left) and the graph of $f'(p)$ (right), the latter is acquired from switching the two axes, effectively ``flipping'' $f(t)$ and removing some vertical segments. Both envelope the same amount of area of the Cartesian plane (shaded).}
        \label{fig:integral}
\end{figure}

\paragraph{Remark} We can also use a standard change of variable in Lebesgue integration, however given that the focus of the paper is not on theoretical analysis, we opt for a more elementary and visual approach.

For some $0 \le p \le 1$, suppose $j \ge 2$ is such $A_{j-1} < (1 - p)|R| \le A_j$. Consider the subpath of $\gamma^*$ from the starting point until an area of $(1 - p)|R|$ is covered, the length of which is $f'_{\gamma^*}(p)$ by definition. Since doubling that subpath gives a tour that covers an area of $(1 - p)|R|$, we have $2^{j-1}\le 2f'_{\gamma^*}(p)$. 

On the other hand, $\overline{\gamma}$ is a concatenation of tours of length no longer than $c2^j$ covering an area of $\overline{A_j} \ge A_j$ for $j = 1, 2, \ldots$. Thus, the distance traveled along $\overline{\gamma}$ by the robot before covering an area of $p$ is no longer than $\sum\limits_{k\le j}c2^k \le 8c2^{j-2}$. This implies, $f_{\overline{\gamma}}(p) \le 8cf_{\gamma^*}(p)$ for all $(1 - p) > \frac{A_1}{|R|}$. Consider that
\begin{align*}
    \int_0^1f_{\overline{\gamma}}'(p)dp &= \int_0^{1 - \frac{A_1}{|R|}}f_{\overline{\gamma}}'(p)dp + \int_{1 - \frac{A_1}{|R|}}^1f_{\overline{\gamma}}'(p)dp.
\end{align*}
We have
\begin{align}
    \label{eqn:bound_multiplicative}
    \int_0^{1 - \frac{A_1}{|R|}}f_{\overline{\gamma}}'(p)dp \le 
    8c\int_0^{1 - \frac{A_1}{|R|}}f_{\gamma^*}'(p)dp \le 8c\int_0^{1}f_{\gamma^*}'(p)dp.
\end{align}
When the fraction of uncovered area is between $1$ and $1 - \frac{A_1}{|R|}$ (meaning the covered area is between 0 and $A_1$), $f_{\overline{\gamma}}'(p)$ is no greater than $2c$, since following along $\overline{\gamma}$ by a distance of $2c$ covers an area of at least $A_1$. As a result
\begin{align}
\label{eqn:bound_constant}
    \int_{1 - \frac{A_1}{|R|}}^1f_{\overline{\gamma}}'(p)dp \le \int_{1 - \frac{A_1}{|R|}}^12cdp = 2c\frac{A_1}{|R|}\le 2c.
\end{align}
Combining the two inequalities \eqref{eqn:bound_multiplicative} and \eqref{eqn:bound_constant} gives
\[E[T\mid\overline{\gamma}] \le 8cE[T\mid\gamma^*] + 2c.\]
\begin{theorem}
    Given a $c$-approximation to the quota lawn mowing problem. Let $\gamma^*$ be an optimal route for minimizing the expected detection time in a polygonal region $R$ under the lawn mowing search mechanism, in pseudopolynomial time, we can compute a route $\overline{\gamma}$ satisfying $E[T\mid\overline{\gamma}] \le 8cE[T\mid\gamma^*] + 2c$ in pseudopolynomial time.
\end{theorem}
\paragraph{Remark} The approximation bound we give includes both a multiplicative and an additive error term. We leave open the problem of removing the additive term. In practice, for a sufficiently large region $R$ which requires a long tour to cover a large area, $\frac{A_1}{|R|}$ should be close to 0.

\subsection{Visibility-based search mechanism}
We show that the idea of recursively doubling the search perimeter in the previous subsection extends also to the visibility-based search mechanism, which models a robot entering and navigating through a room with obstacles (e.g., pieces of furniture). The robot is equipped with an omnidirectional camera for sensing, and it detects the stationary target when vision of the target is established. One difference between the visibility-based search mechanism and the lawn mowing search mechanism is that the robot cannot go through obstacles (the notion of visibility is not well-defined from within obstacles), nevertheless this does not prevent us from employing the same idea as before, thanks to a known result on partial visibility-based search. 

In a polygonal region $R$, two points $x$ and $y$ \textit{see} each other if the line segment $xy$ does not intersect the exterior of $R$. The \textit{visibility region} of $x$, denoted $V(x)$, is the set of all points that $x$ sees. We also define the visibility region of a subset $S\subset R$ (for example, a route), which we denote by $V(S)$, to be the union of the visibility regions of all points in $S$ (refer to Figure~\ref{fig:coverage_watchman}). Efficient algorithms for computing the visibility region of points and segments in polygonal domains are known~\cite{guibas1986linear,heffernan1995optimal}.

\begin{figure}[H]
        \centering
        \includegraphics[width=0.75\textwidth]{
        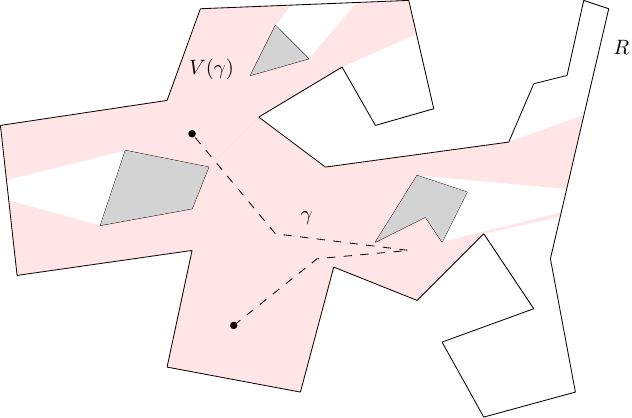}
        \caption{The area of coverage $V(\gamma)$ (red) by a robot with visibility-based detection.}
        \label{fig:coverage_watchman}
\end{figure}

Given a polygonal domain $R$. We want to compute a route $\gamma \subset R$ that minimizes the expected amount of time before the robot sees the target, which we also denote $E[T\mid\gamma]$ and define similarly to that of the lawn mowing mechanism. For visibility-based detection, it has been known that minimizing the expected time until detection is NP-hard in a polygon with holes~\cite{sarmiento2003efficient}. We strengthen the hardness result in the following theorem.

\begin{theorem}
    For visibility-based search, minimizing $E[T\mid\gamma]$ is NP-hard, even in a simple polygon.
\end{theorem}
\begin{proof}
    We once again show NP-hardness by a reduction from the minimum latency problem, but for a weighted tree, which is known to be NP-hard \cite{sitters2002minimum}. Given a weighted tree, create a straight-line embedding of the tree using any graph drawing algorithm (see~\cite{tamassia2013handbook} for a comprehensive survey about graph drawings) in polynomial time. Then, construct a simple polygon from the embedding, similar to the construction used in the proof of Theorem \ref{thm:min_E_NP_hard}. A tour within the polygon minimizing expected detection time yields a minimum latency tour.\hfill$\blacksquare$
\end{proof}

\subsubsection{The budgeted watchman route problem}
The budgeted watchman route problem was first introduced in \cite{huynh2024optimizing}. Given a polygon $P$ and a budget $B > 0$, compute a route $\gamma$ such that $|\gamma| \le B$ and the area of the visibility region of $\gamma$, $|V(\gamma)|$ is maximized.

The budgeted watchman route problem is (weakly) NP-hard, the following approximation algorithm is known~\cite{huynh2024optimizing}, which we require as a subroutine. 

\begin{lemma}
\label{lem:bwrp}
    \cite[Theorem 11]{huynh2024optimizing} Let $P$ be a simple polygon with $n$ vertices. Given a starting point $s$, in fully polynomial time (in $n$ and $\displaystyle\frac{1}{\varepsilon})$, we can compute a route of length no longer than $(1 + \varepsilon)B$ that passes through $s$ and sees an area no smaller than any route of length $B$ can see.
\end{lemma}
\begin{theorem}
Let $\gamma^*$ be an optimal route for minimizing the expected detection time in a simple polygon $P$ under the visibility-based search mechanism. For $B = 2^j$ where $j= 1, \ldots, O(\log N)$ (recall that $N$ is the number of unit-sized integer pixels intersecting $P$), run the approximation algorithm for the budgeted watchman route problem in Lemma \ref{lem:bwrp} and concatenate the approximate tours into one, which we denote by $\overline{\gamma}$. Then $E[T\mid\overline{\gamma}] \le (8 + 8\varepsilon)E[T\mid\gamma^*] + (2 + 2\varepsilon)$ for any $\varepsilon > 0$.
\end{theorem}
\begin{proof}
    The analysis follows through the same way as for the lawn mowing search mechanism. The running time is equal to that of the approximation algorithm to the budgeted watchman problem multiplied by $O(\log N)$.\hfill$\blacksquare$
\end{proof}

\section{Simulation Results}
Algorithms developed to guarantee theoretical bounds in quota TSP-type problems are usually far too complicated and incur impractical running times. Thus, we complement our theoretical analysis with some experimental heuristics that are relatively straightforward to implement. Figure~\ref{fig:polygon_11} shows our test polygons with holes in which we want to search for a target with a square scanner and rectilinear movement. 

\begin{figure}[H]
    \hspace{-1.5cm}\includegraphics[width=0.6\textwidth]{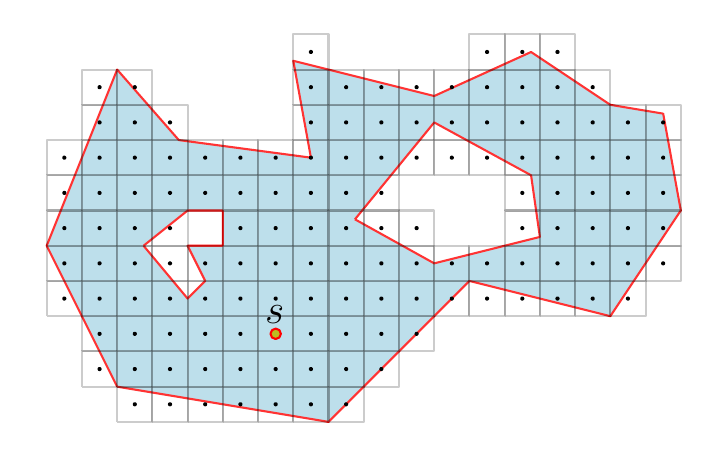}\includegraphics[width=0.72\textwidth]{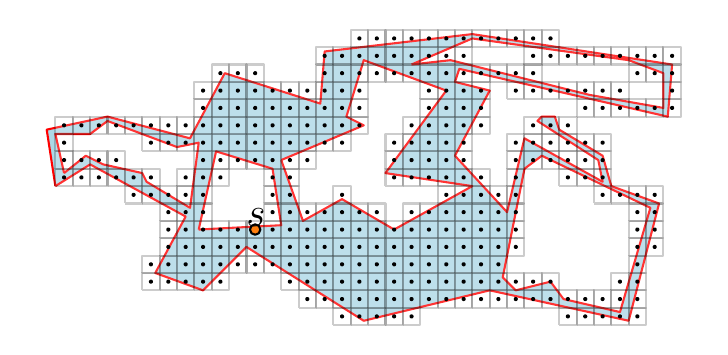}
    \caption{The test polygonal regions, with the starting point $s$.}
    \label{fig:polygon_11}
\end{figure}

Our code is written in Python~\cite{van1995python}, we use the following packages: Shapely~\cite{shapely2007} for processing and handling the geometry of the domain and NetworkX~\cite{hagberg2008exploring} for handling the graph-theoretical components of the heuristic. Given below is a high-level description of what we call the exponential tree heuristic.
\begin{algorithm}[H]
\caption{Exponential Tree Heuristic}\label{alg:heuristic}
$j = 0$\;

visited = \{\}\;

\textbf{while} $|\text{visited}| \le N$
{

    \hspace*{0.5cm}Initialize the tree $T$ = [starting center point]\;
    
    \hspace*{0.5cm}\textbf{while} $|T| \le \min\{2^j, N\}$
    {
    
    \hspace*{1cm}$T\leftarrow T\cup v$ where $v$ is adjacent to a center point in $T$ and $r(v) \rightarrow \max$\;
    
    \hspace*{1cm}Traverse the TSP tour on the nodes in $T\setminus\{v\mid v\ne \text{starting point}, v\in \text{visited} \}$\;
    
    \hspace*{1cm} visited = visited $\cup$ $T$ \;
    
    \hspace*{1cm}$j \leftarrow j + 1$\;
    }
}
\end{algorithm}

The main idea is to greedily pick a subset of center points that maximize the total reward, while ensuring the length of the shortest tour through them is no larger than $2^j$ for $j = 0, 1, \ldots$ The center points should induce a connected graph in $\mathcal{G}$, and their spanning tree should have cost no greater than $2^j$, making the TSP tour on them no longer than $2^{j+1}$. By discarding the center points already visited (by previous tours of smaller maximum sizes) and computing only the TSP tour on the remaining ones, the running time as well as the detection time is improved by a margin.

It is quite infeasible to find an optimal minimum latency tour by brute force, even for small instances, due to the factorial number of possible tours. We thus compare the mean detection time of the route generated by our heuristic with that of a heuristic designed to minimize the total latency tour on all the center points. First, note that it suffices, for purposes of approximating the minimum latency tour within a factor of $(1 + \varepsilon)$, to concatenate a number of TSP paths with a geometrically decreasing sequence of number of nodes ~\cite{arora2003approximation}. Thus, we find the overall TSP tour on $N$ center points, replace the subpath on the first $\displaystyle\left\lfloor\frac{\varepsilon N}{1 + \varepsilon}\right\rfloor$ center points with the TSP path on those points, then repeat with the next $\displaystyle\left\lfloor\frac{\varepsilon N}{(1 + \varepsilon)^2}\right\rfloor, \left\lfloor\frac{\varepsilon N}{(1 + \varepsilon)^3}\right\rfloor, \ldots$ points. We set $\varepsilon = 0.01$ to get close to the minimum latency tour.

We perform our experiment with 1000 different random generations of the location of the target, the results are given in Table~\ref{tab1}.

\begin{table}
\caption{Detection time comparisons between the exponential tree heuristic and the minimum latency heuristic.}\label{tab1}\centering
\begin{tabular}{|l|l|l|l|}
\hline
1st test polygonal region &  mean & standard deviation & running time\\
\hline
Exponential Tree Heuristic & 112.081 & 70.333 & 1.308 \\ 
 Minimum Latency Heuristic & 69.577 & 42.764 & 8.980 \\
\hline
\end{tabular}
\vspace{0.5cm}

\begin{tabular}{|l|l|l|l|}
\hline
2nd test polygonal region &  mean & standard deviation & running time\\
\hline
Exponential Tree Heuristic & 233.745 & 151.886 & 4.843 \\ 
 Minimum Latency Heuristic & 183.167 & 119.555 & 52.360 \\
\hline
\end{tabular}
\end{table}

The exponential tree heuristic seems to yield average detection time within a constant factor of that of the minimum latency heuristic, but has faster running time. This demonstrates that gradually expanding the search area is an effective strategy for finding a target within a reasonable expected time. The search routes generated by the exponential tree heuristic are given in Figures~\ref{fig:route_heuristic_2} and~\ref{fig:route_heuristic_1}.

\begin{figure}[h]
\hspace{-2.5cm}\includegraphics[width=0.65\textwidth]{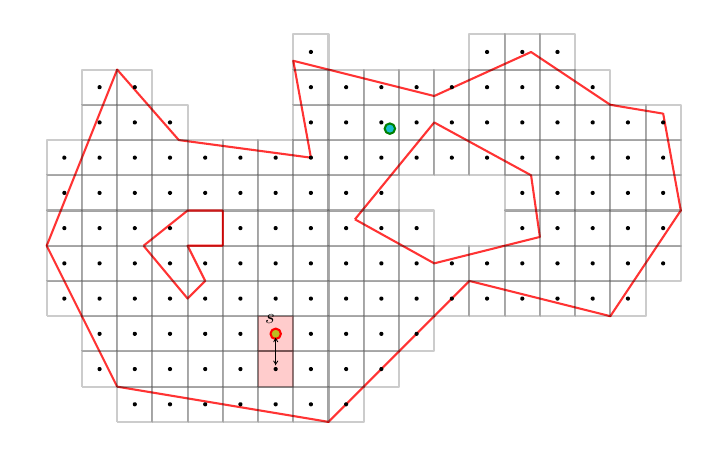}\includegraphics[width=0.65\textwidth]{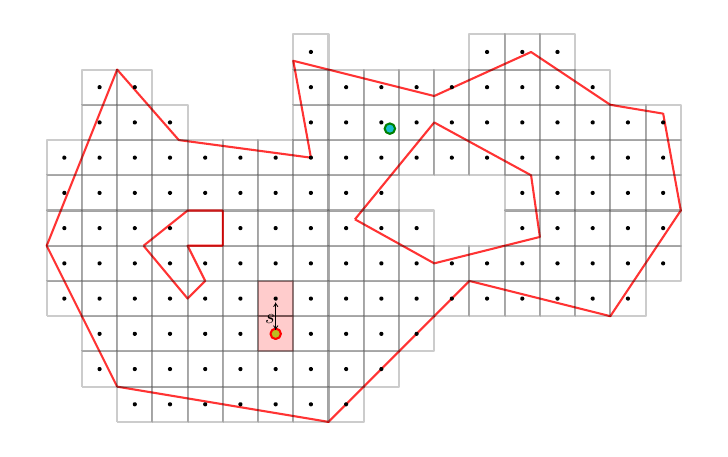}\\
\hspace*{-2.5cm}\includegraphics[width=0.65\textwidth]{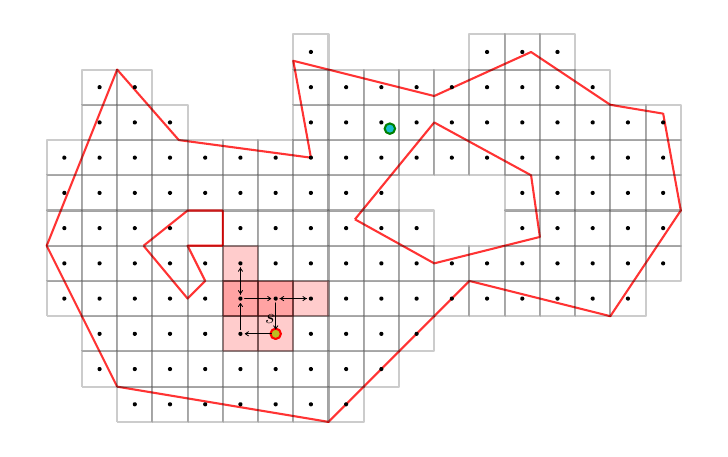}\includegraphics[width=0.65\textwidth]{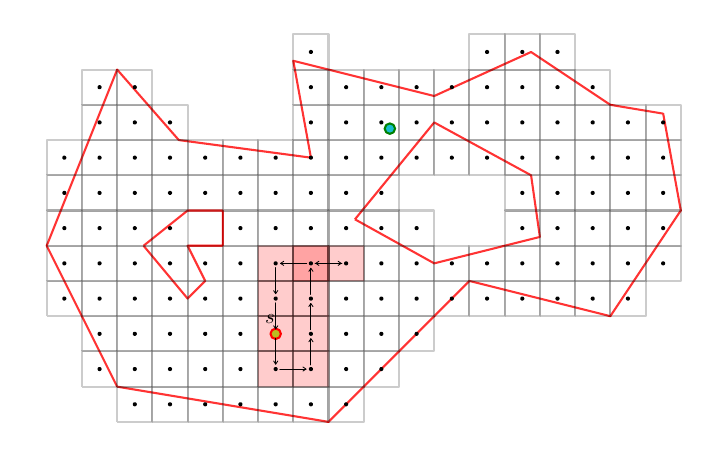}\\
\hspace*{-2.5cm}\includegraphics[width=0.65\textwidth]{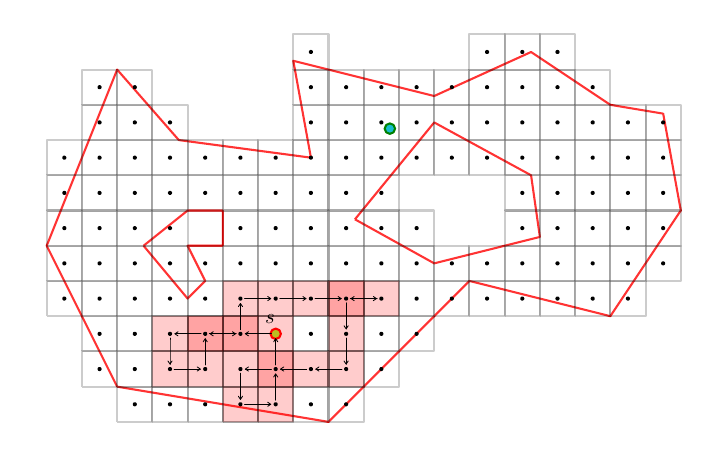}\includegraphics[width=0.65\textwidth]{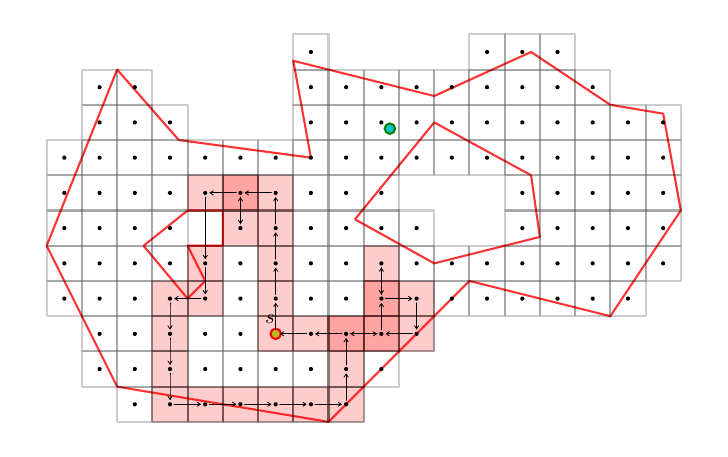}\\
\hspace*{-2.5cm}\includegraphics[width=0.65\textwidth]{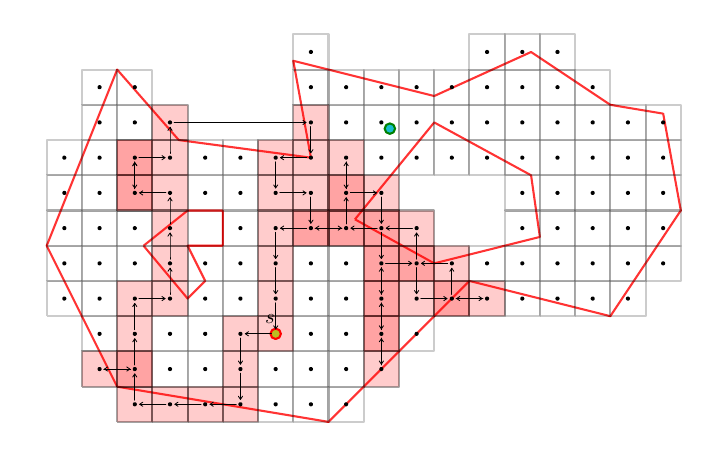}\includegraphics[width=0.65\textwidth]{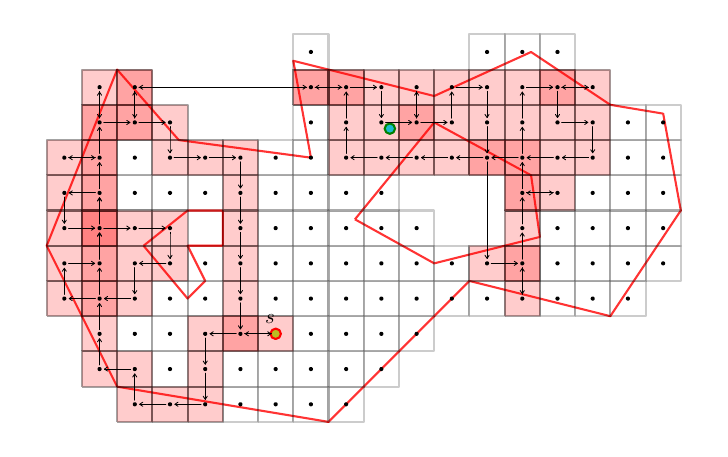}\\
    \caption{The search route, which consists of tours of increasing lengths, generated for the first test polygonal region by the exponential tree heuristic. The target is drawn in green.}
    \label{fig:route_heuristic_2}
\end{figure}

\begin{figure}[h]
\hspace{-2.5cm}\includegraphics[width=0.7\textwidth]{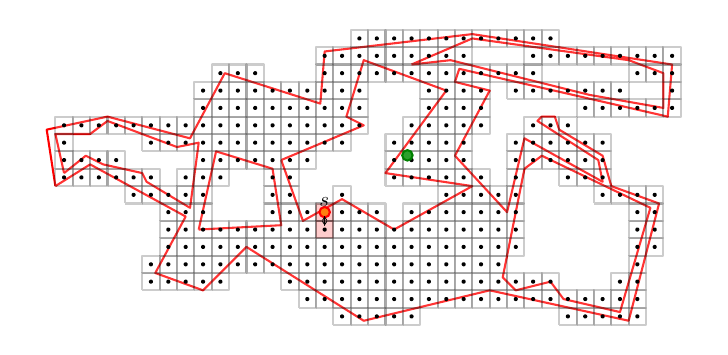}\includegraphics[width=0.7\textwidth]{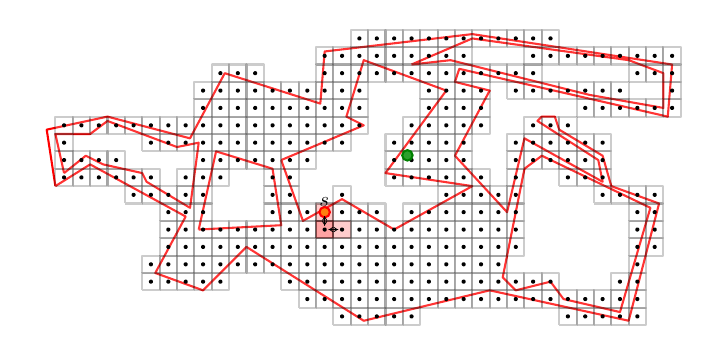}\\
\hspace*{-2.5cm}\includegraphics[width=0.7\textwidth]{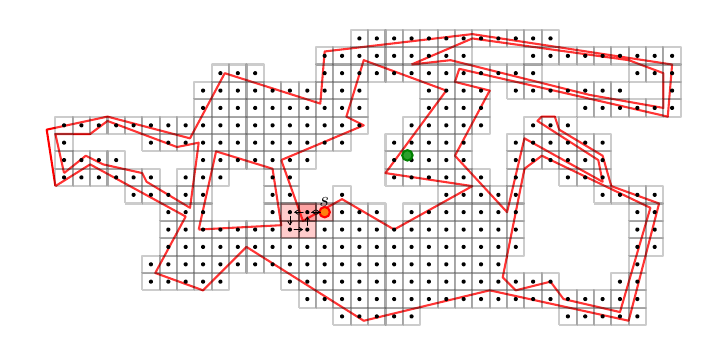}\includegraphics[width=0.7\textwidth]{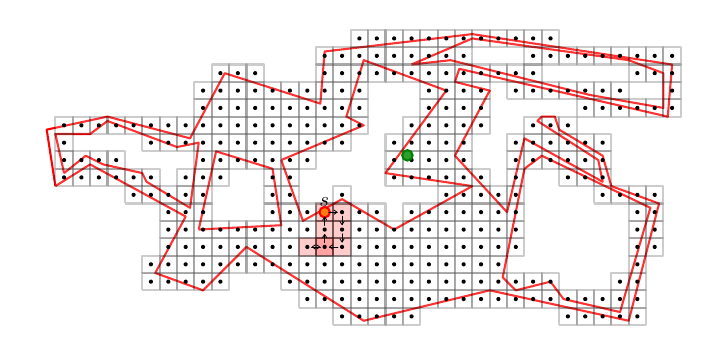}\\
\hspace*{-2.5cm}\includegraphics[width=0.7\textwidth]{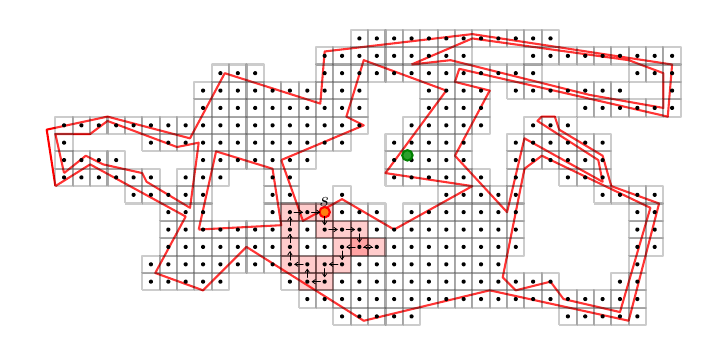}\includegraphics[width=0.7\textwidth]{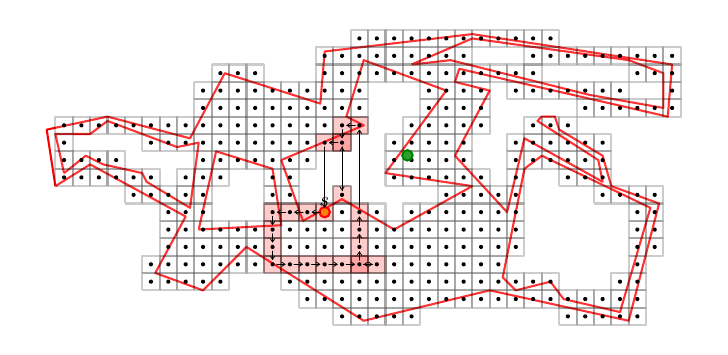}\\
\hspace*{-2.5cm}\includegraphics[width=0.7\textwidth]{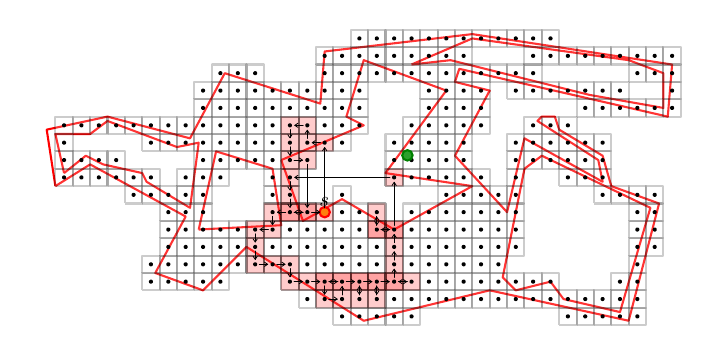}\includegraphics[width=0.7\textwidth]{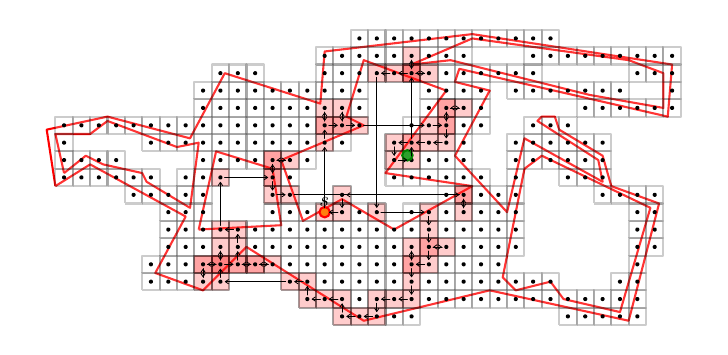}\\
    \caption{The search route, which consists of tours of increasing lengths, generated for the second test polygonal region by the exponential tree heuristic. The target is drawn in green.}
    \label{fig:route_heuristic_1}
\end{figure}

\begin{credits}
\subsubsection{\ackname} This work is partially supported by the National Science Foundation (CCF-2007275). The author would like to express his gratitude to his advisor, Joseph S. B. Mitchell for incredibly helpful and encouraging discussions on the paper.

\subsubsection{\discintname} The author has no competing interests to declare that are
relevant to the content of this article.
\end{credits}
%
%
%
%

\bibliographystyle{splncs04}
\bibliography{refs.bib} 
\end{document}